\documentclass{article}

\usepackage{microtype}
\usepackage{graphicx}
\usepackage{booktabs} 
\usepackage{hyperref}

\usepackage[caption=false]{subfig}
\usepackage[accepted]{icml2020}

\usepackage[caption=false]{subfig}
\usepackage{amsfonts}                   
\usepackage{amsmath}                    
\usepackage{amssymb}                    
\usepackage{amsthm}

\newcommand{\R}{\mathbb{R}}
\newcommand{\opn}[1]{\|#1\|_{p\rightarrow q}}

\newcommand{\X}{\mathcal{X}}
\newcommand{\Y}{\mathcal{Y}}

\newtheorem{theorem}{Theorem}
\newtheorem{corollary}{Corollary}
\newcommand{\class}{C} 
\newcommand{\scoreclassset}{\mathcal{F}}

\begin{document}

\twocolumn[
\icmltitle{Input Hessian Regularization of Neural Networks}



\icmlsetsymbol{equal}{*}

\begin{icmlauthorlist}
\icmlauthor{Waleed Mustafa}{tuk}
\icmlauthor{Robert A. Vandermeulen}{tub}
\icmlauthor{Marius Kloft}{tuk}

\end{icmlauthorlist}

\icmlaffiliation{tuk}{Computer Science Department, TU Kaiserslautern}
\icmlaffiliation{tub}{Machine Learning Group, TU Berlin}

\icmlcorrespondingauthor{Waleed Mustafa}{mustafa@cs.uni-kl.de }

\icmlkeywords{Machine Learning, ICML, Second Order Regularization, Adversarial Example Defenses}

\vskip 0.3in
]



\printAffiliationsAndNotice{}  

\begin{abstract}
Regularizing the input gradient has shown to be effective in promoting the robustness of neural networks. The regularization of the input's Hessian is therefore a natural next step. A key challenge here is the computational complexity. Computing the Hessian of inputs is computationally infeasible. 
In this paper we propose an efficient algorithm to train deep neural networks with Hessian operator-norm regularization. 
We analyze the approach theoretically and prove that the Hessian operator norm relates to the ability of a neural network to withstand an adversarial attack. We give a preliminary experimental evaluation on the MNIST and FMNIST datasets, which demonstrates that the new regularizer can, indeed, be feasible and, furthermore, that it increases the robustness of neural networks over input gradient regularization. 
\end{abstract}

\section{Introduction}
Allowing for the automation of many previously impossible tasks, deep learning has brought about tremendous advances in machine learning. However, the lack of robustness of deep neural networks (DNNs) has caused some concern regarding their deployment, particularly in security or safety-critical applications. One such concern is the existence of adversarial examples \cite{Szegedy2013IntriguingPO}---samples that have been imperceptibly modified to trick a deep learning system to behave incorrectly. Many methods have been proposed to generate such examples \cite{Carlini2017TowardsET,Kurakin2016AdversarialML,Athalye2018ObfuscatedGG,Moosavi-Dezfooli2015} or defend neural networks against them \cite{Hein2017FormalGO,singla2020curvaturebased,moosavi2019robustness,Madry2017TowardsDL}.


A line of research suggests that the regularization of input gradients helps promote robustness \cite{finlay2019scaleable,ross2018improving,Hein2017FormalGO}. To better theoretically understand the robustness of neural networks against adversarial attacks, \citet{Hein2017FormalGO} introduced a theoretical bound on the robustness of a network in terms of the norm of its gradient around a test sample. A large (input) gradient at a sample implies that it is possible to change the network output by only slightly manipulating the input. Using this theoretical insight, they propose a defense in the form of a regularizer that promotes the model's gradient norms to be small. Their regularizer assumes that the network is locally well-approximated by a linear function. This approximation can be poor when a function is highly non-linear, as it is the case with neural networks. Second-order approximations can be more suitable here and improve robustness. This was empirically supported in \citet{fawzi2017classification}, where the authors showed that adversarial directions can highly correlate with high principal curvatures (i.e., Hessian eigenvectors that correspond to high eigenvalues). 

In this paper we propose to use second-order derivatives to better approximate the behavior of the classifier around a sample. We prove guarantees in the form of lower bounds on the robustness of neural networks in terms of second-order derivatives. 
Motivated by our theoretical results, we propose \textit{Cross-H\"older Regularization}, a second order regularizer that penalizes input Hessians with large eigenvalues. While a naive approach to computing the gradient of the largest absolute eigenvalue of a Hessian is computationally expensive, we introduce an efficient procedure to compute it. We present a preliminary empirical evaluation on the MNIST and FMNIST datasets, which showcases that our method can outperform gradient-based regularization on robustness to adversarial samples.

The paper is organized as follows. In the remainder of this section we introduce the notation. In Section \ref{sec:theorem}  we introduce a new theory that relates adversarial robustness to second-order differential properties of neural networks. In Section \ref{sec:regularizer} we introduce a new regularizer as well as a method for computing the gradient of the operator norm of a Hessian. Finally in Section \ref{sec:exp} we present an empirical analysis. Note that related work is discussed in Appendix~\ref{sec:background}.

\paragraph{Notation and Problem Setup
}Let $\X \subset \R^d$ be a feature space and $\Y = \{1, \ldots, K\}$ a label space. Denote the training set by $\{(x_i,y_i)\}_{i=1}^N$. We are interested in classifiers that are derived from a scoring function ${\bf f} : \X \rightarrow \R^K$ from which we classify via $\class_{\bf f}(x) := \arg\max_i {\bf f}_i(x)$. Training is therefore achieved by selecting a scoring function from a  family of functions $\scoreclassset$ indexed by real-valued parameters $\theta \in \R^D$, with the goal to minimize the objective $\Sigma_{i=1}^N \ell({\bf f}_\theta(x_i),y_i)$, where $\ell: \R^K \times \Y \rightarrow \R^+$ is a loss function. For the rest of the paper we use ${\bf f}(x)$ instead of ${\bf f}_\theta(x)$ to avoid notational clutter. For a matrix   $H \in \R^{d\times d}$, the operator norm $\opn{H}$ is defined as:  $\opn{H} := \sup \{\|Hx\|_q: \|x\|_p = 1\}.$ 
Our goal is to train classifiers that are robust under adversarial attacks. That is, if a classifier $\class_{\bf f}$ predicts a class for an input $x$, it should stay consistent for similar examples $x'$. Formally, for a \textit{robustness level} $\epsilon > 0$, we require $\class_{\bf f}(x) = \class_{\bf f}(x+\delta)$, for $\|\delta\|_p \le \epsilon$.

\section{Robustness and Second Order Derivatives}
\label{sec:theorem}
In this section, we prove a theoretical result that relates the robustness of a classifier to the input Hessian operator norm. Our first result is akin to Theorem 1 in \citet{Hein2017FormalGO}, where the authors proved that robustness is controlled by the norm input gradient of a margin loss in a local $\ell_p$ ball around it. We extend their theory to control the robustness of classifier by the Hessian operator norm in an $\ell_p$ ball around an input example. Therefore, we show an explicit connection between the robustness and curvature. This connection allows us to derive our second order regularizer. To avoid notational clutter, we present our result for the binary classifier case. First we review the result by \citet{Hein2017FormalGO}.  
\begin{theorem}\label{thm:hein}

     Let $x \in \R^d$ and ${\bf f}: \R^d \rightarrow \R^2$ be a binary-classifier scoring function with continuously differentiable components. Further let ${\bf f}_0(x) > {\bf f}_1(x)$ and define $f = {\bf f}_0 - {\bf f}_1$. For $\delta \in \mathbb{R}^d$ and $p,q$ such that $\frac{1}{p} + \frac{1}{q} = 1$, if ${\bf f}_1(x+\delta) \ge {\bf f}_0(x+\delta)$, then the following bound holds
    \begin{align}
        \label{eq:hein}
        \|\delta\|_p \ge \max_{R>0}\min\bigg\{R,\frac{f(x)}{\max_{\|\gamma\|_p\le R}\|\nabla f(x+\gamma)\|_q }\bigg\}.
    \end{align}
\end{theorem}
The above bound implies that, if $f$ is not steep around $x$ (i.e. $\nabla f(\delta+x)$ is small), then one must move a large distance away from $x$ to find an adversarial example (i.e. $\delta$ is large).
This bound was shown to be tight in \citet{Hein2017FormalGO} when considering linear classifiers.

Note here that the dependence on curvature is only implicit. Indeed, if $f$ is locally highly curved, then $\max_{\|\gamma\|_p\le R}\|\nabla f(x+\gamma)\|_q$ is large. In order to design a curvature regularizer, we introduce the following result which makes the dependence on the curvature explicit. 
\begin{theorem}
    \label{thm:tight}
    Let $x \in \R^d$ and ${\bf f}: \R^d \rightarrow \R^2$ be a binary-classifier scoring function with continuously twice differentiable components. Further let ${\bf f}_0(x) > {\bf f}_1(x)$ and define $f = {\bf f}_0 - {\bf f}_1$. $H_f:\R^d\rightarrow \R^{d\times d}$  denotes the Hessian of $f$. For $\delta \in \mathbb{R}^d$ and $p,q$ such that $\frac{1}{p} + \frac{1}{q} = 1$, if ${\bf f}_1(x+\delta) \ge {\bf f}_0(x+\delta)$, then the following bound holds:
    \begin{align}
        \|\delta\|_p \ge \max_{R>0}\min\left \{R,\frac{f(x)}{\|\nabla f(x)\|_q+ \frac{R}{2}K_R(f,x)}\right \},
    \end{align}
    where $K_R(f,x) = \max_{\|\gamma\|_p\le R}K(f,x+\gamma)$, for $K(f,x):= \opn{H_f(x)}$.
\end{theorem}

The proof of the above theorem is deferred to Appendix \ref{sec:proofs}. The theorem suggests that, to control the robustness of a classifier, we must control the local curvature around input points. Again, this requirement is on $f$, not the individual components ${\bf f}_i$. Therefore, the classifier components can have a large local curvature as far as the relative curvature is low, which is a less strict requirement. The dependence on the curvature is explicit in this result through the input Hessian operator norm of the margin function $f$ locally around the input point. Since this bound is equivalent to the bound \eqref{eq:hein} in the case of linear scoring functions, it borrows its tightness. 


In the next corollary, we extend our result to a multi-class classifier. It is straightforward to do so. If a classifier predicts a class $t$, then define $f^{(j)} = {\bf f}_t - {\bf f}_j$. We then apply theorem \ref{thm:tight} to $f^{(j)}$, for $j \ne t$, and take the minimum of resulting lower bounds.

\begin{corollary}
    Let $x \in \R^d$ and ${\bf f}: \R^d \rightarrow \R^K$ be a multi-class classifier scoring function with continuously twice differentiable components. Further let $t = \arg\max_i {\bf f}_i(x)$ and define $f^{(j)} = {\bf f}_t - {\bf f}_j$ for $j \ne t$. For $\delta \in \mathbb{R}^d$ and $p,q$ such that $\frac{1}{p} + \frac{1}{q} = 1$, if $\max_{j \ne t}{\bf f}_j(x+\delta) > {\bf f}_t(x+\delta)$, then $\|\delta\|_p$ is bounded below by:
    \begin{align}\label{eq:reg2_bound2}
        \max_{R>0}\min\left \{R,\min_{j \ne t}\frac{f^{(j)}(x)}{\|\nabla f^{(j)}(x)\|_q+ \frac{R}{2}K_R(f^{(j)},x)}\right \}.
    \end{align}
\end{corollary}
    
\section{Cross H\"older Regularizer} \label{sec:regularizer}
In this section, we introduce a new technique for adversarial defending. The goal is to define a regularizer that improves the robustness of multi-class classifiers. This regularizer is based on~\eqref{eq:reg2_bound2} introduced in the last section. We start by deriving the regularization term to maximize the lower bound. Directly doing so is computationally impractical. Therefore, we reformulate it, to make it amenable for training.  

Our goal is to train a classifier so that for an adversarial example, $x+\delta$, the norm of $\delta$ must be large, i.e. adversarial examples must have a large amount of distortion. To that end, we train our classifier in such a way that the right-hand side of \eqref{eq:reg2_bound2} is large at test samples. It suffices to maximize the term for each $j$ independently, that is,
\begin{align}
    \label{eq:obj}
    \max_\theta\frac{f^{(j)}(x)}{\left\|\nabla f^{(j)}(x) \right\|_q + \frac{R}{2} K_R(f,x)}.
\end{align}
It was argued in \citet{Hein2017FormalGO} that the cross entropy loss will naturally drive $f^{(j)}(x)$ to be large.
Therefore we turn our attention to the denominator of \eqref{eq:obj}. Minimizing the gradient term is standard \cite{drucker1992improving,Hein2017FormalGO,ross2018improving}. The Hessian term is, however, hard to minimize. The first issue is the maximization operator $\max_{\|\gamma\|_p \le R}$, which makes evaluation of this term computationally expensive. To alleviate this issue, we assume that $\opn{H_{f^{(j)}}(x)} \approx \max_{\|\gamma\|_p \le R}\opn{H_{f^{(j)}}(x+\gamma)}$. This can be achieved if we assume that the function $f$ is locally well-approximated by a quadratic function. 

Our final objective is thus to make $\left\|\nabla f^{(j)}(x)\right\|_q + \frac{R}{2} \opn{H_{f^{(j)}}(x)}$ small for all $j \ne y$, where we now treat $R$ as a tuning parameter. Relaxing this regularizer a bit and taking a summation of all possible classes we arrive at our proposed training objective, where we enforce our network to be robust around training samples. We define our loss $L\left(\{x_i,y_i\}_{i=1}^n\right)$ to be
\begin{align}
\label{eq:loss}
\begin{split}
    &\sum_{i=1}^n \ell(y_i,x_i) + \\
     &\lambda_1 \sum_{i=1}^n\sum_{j=1,j\ne y_i}^K \left(\|\nabla f^{(j)}(x_i)\|_q +\lambda_2 K(f^{(j)},x_i)\right),
\end{split}
\end{align}
where $\ell(.,.)$ is the cross-entropy loss, and $\lambda_1$ and $\lambda_2$ are regularization parameters. The last two terms of~\eqref{eq:loss} together form the \emph{Cross-H\"older Regularizer}.
Notice that when $\lambda_2 =  0$ we get the Cross-Lipschitz regularizer.

\section{Training with Hessian Operator norm}
To minimize the new objective \eqref{eq:loss}, we will use stochastic gradient descent. Thus we need to evaluate the gradient of the objective. Obtaining the gradients needed for SGD is usually done through \textit{automatic differentiation} \cite{griewank2008evaluating}, as implemented in popular machine learning toolkits (e.g., Tensorflow \cite{Abadi2016TensorFlowAS}). It is sufficiently efficient to compute the gradient of the loss and gradient norm. Computing the gradient of the operator norm of the Hessian is, however, very computationally inefficient for algorithmic differentiation. Obtaining the Hessian is very costly for large neural network architectures, let alone computing its operator norm and computing the gradient of the combined operations with respect to the model parameters.
 
We now describe how to efficiently calculate the gradient of the Hessian operator norm. In what follows we use $H_\theta$ to denote $H_{f^{(j)}}(x)$. Define a function $g: \R^{d} \times \R^D \rightarrow \R^+$ as 
\begin{equation}
  \label{eq:op}
 g(\theta,v) =  \|H_\theta v\|_q,  
\end{equation}
where $\theta$ is the model parameters and $v \in \{ v\in \R^d: \|v\|_p = 1\}$. Therefore, the operator norm is
\begin{align*}
    \opn{H_\theta} = \max_{ \|v\|_p = 1} g(\theta, v).
\end{align*}
Let $v^*$ be the vector attaining the maximum of \eqref{eq:op}. If $v^*$ is a unique maximizer, then, by Danskin's theorem \cite{danskin2012theory}, we have
\begin{align*}
    \nabla_{\theta} \max_{\|v\|_p=1}g(\theta,v) = \nabla_{\theta} g(\theta,v^*).
\end{align*}
It was also shown in \citet{Madry2017TowardsDL} that, even if $v^*$ is not a unique maximum, the direction $-\nabla_{\theta} g(\theta,v^*)$ is still a descent direction. Therefore it suffices to compute the gradient at the maximum $v^*$ for training.

There still remains the problem of computing both $v^*$ and $\nabla_\theta \|H_\theta v^*\|_q$. One can find $v^*$ by optimizing \eqref{eq:op}. It is, however, again the problem of computing the Hessian matrix $H_\theta$. Fortunately, the Hessian matrix appears above only multiplied by a vector. This allows us to use an efficient algorithm to obtain a computational graph for Hessian vector multiplication \cite{Pearlmutter1994}. With this algorithm one can use algorithmic differentiation to obtain both $\nabla_\theta \|H_\theta v^*\|_q$ and $\nabla_v \|H_\theta v\|_q$\footnote{When $\|H_\theta v\|_q$ is differentiable as a function of $v$ which is the case for $q=2$. } with a computational expense on the same order as backpropagation \cite{griewank2008evaluating,Pearlmutter1994}. In a nutshell, our approach reduces to solving
\begin{align*}
    \min_\theta&\max_{v_{ij}} 
    \sum_{i=1}^n \ell(y_i,x_i)+& \\
    &\lambda_1 \sum_{i=1}^n\sum_{j\ne y_i}^K \left(\|\nabla f^{(j)}(x_i)\|_q +\lambda_2 \| H_{f^{(j)}}(x_i)v_{ij}\|_q \right),
\end{align*}
for $\|v_{ij}\|_p = 1$. Algorithm \ref{alg:1} summarizes our training approach. The inner loop is used to estimate $v^*$ which is then used in the last step to estimate the gradient of the objective. In our experiments, we found setting $T=10$ and $o = 0.1$ in the algorithm suffices to get a reasonable approximation to the Hessian operator norm. It is worth noting that thanks to the hessian vector multiplication algorithm \cite{Pearlmutter1994}, computing the gradient enjoys a linear running time.  


\begin{algorithm}[tb]
\label{alg:1}
 \caption{Hessian Norm Regularization training}
   \label{alg:example}
\begin{algorithmic}
   \STATE {\bfseries Input:} Training data $\{(x_i,y_i)\}_{i = 1}^N$, a model ${\bf f}$, initialized parameters $\theta$, number of epochs $J$, operator norm optimization iterations $T$, hyper parameters $\lambda_1, \lambda_2$, and learning rates $o$ and $l$.
  \FOR { $j = 1$ {\bfseries to} $J$}
   \STATE Sample a mini batch $\{(x_i,y_i)\}_{i=1}^b$ of size $b$.
   \STATE   Initialize $v_{ij}$s i.i.d. from $\mathcal{N}(0,I)$ defined on $\R^d$, for $i\in {1, \ldots, b}$ and $j \in {1, \ldots, K-1}$.
   \FOR{$t=1$ {\bfseries to} $T$}
   \STATE Update $v_{ij}$s for each $i$ and $j$ using:
    \begin{align*}
      v_{ij}  &= v_{ij} + o*\nabla_{v_{ij}} \| H_{f^{(j)}}(x_i)v_{ij}\|_q \\
      v_{ij}  &= \frac{v_{ij}}{\|v_{ij}\|_p} 
     \end{align*}
   \ENDFOR
   \STATE Update parameters $\theta$ by:
      \begin{align*}
         \theta =   \theta &+\sum_{i=1}^b \sum_{j\ne y_i}^{K} l*\nabla_\theta \bigg ( \ell(x_i,y_i)      \\
         & +\lambda_2\|\nabla f^{(j)}(x_i)\|_q + \lambda_1 \lambda_2 \|H_{f^{(j)}}(x_i)v_{ij}\|_q \bigg )
     \end{align*}
   \ENDFOR
\end{algorithmic}
\end{algorithm}
\section{Experiments}
\label{sec:exp}

In this section, we present some first experimental results for evaluating the adversarial robustness of models trained with the Cross-H\"older regularizer. 
We test the robustness on two datasets, namely MNIST and FMNIST. As baseline defenses, we use the gradient based approach of Cross-Lipschitz \cite{Hein2017FormalGO}. We also compare against adversarial training \cite{Madry2017TowardsDL} as it is considered state of the art defense \cite{Athalye2018ObfuscatedGG}. The detailed experimental setup and further results are presented in Appendix \ref{sec:det_exp}.

Evaluating the robustness is a hard problem due to the high non-convexity of the attack optimization problem \cite{Madry2017TowardsDL}. Many defenses claim the robustness against weak attacks while they fail to be robust against harder attacks \cite{Athalye2018ObfuscatedGG,carlini2019evaluating}. Therefore, to evaluate the robustness of models trained with a given defense mechanism, it is important to use the strongest available attacks. We use the recommended PGD attack with multiple iterations and multiple random initializations \cite{carlini2019evaluating}. Since, in adversarial training, an attack is used in the training procedure we use also a different attack in the evaluation phase \cite{carlini2019evaluating}. We use two PGD attacks with 50 iterations and 10 random restarts. The first one uses the cross-entropy loss as objective, as in adversarial training. The second one uses the objective of Carlini and Wagner (CW)  \cite{carlini2019evaluating,Athalye2018ObfuscatedGG}. As a robustness metric, we report on adversarial accuracy, that is, the accuracy after we have applied a given attack. We report the worst-case results of the two attacks, which is the minimum adversarial accuracy of the two attacks.  We report the adversarial accuracy for different robustness levels $\varepsilon$. 
Figure \ref{fig:mnist_res} depicts the worst-case accuracy of the three defenses on the MNIST dataset. The figure shows that all three defenses are quite effective up to robustness level 4.0 after which Cross-H\"older outperforms Cross-Lipschitz.    
 \begin{figure}[h!]
 \centering
  
\label{fig:mnist_worst}
\includegraphics[width=0.4\textwidth]{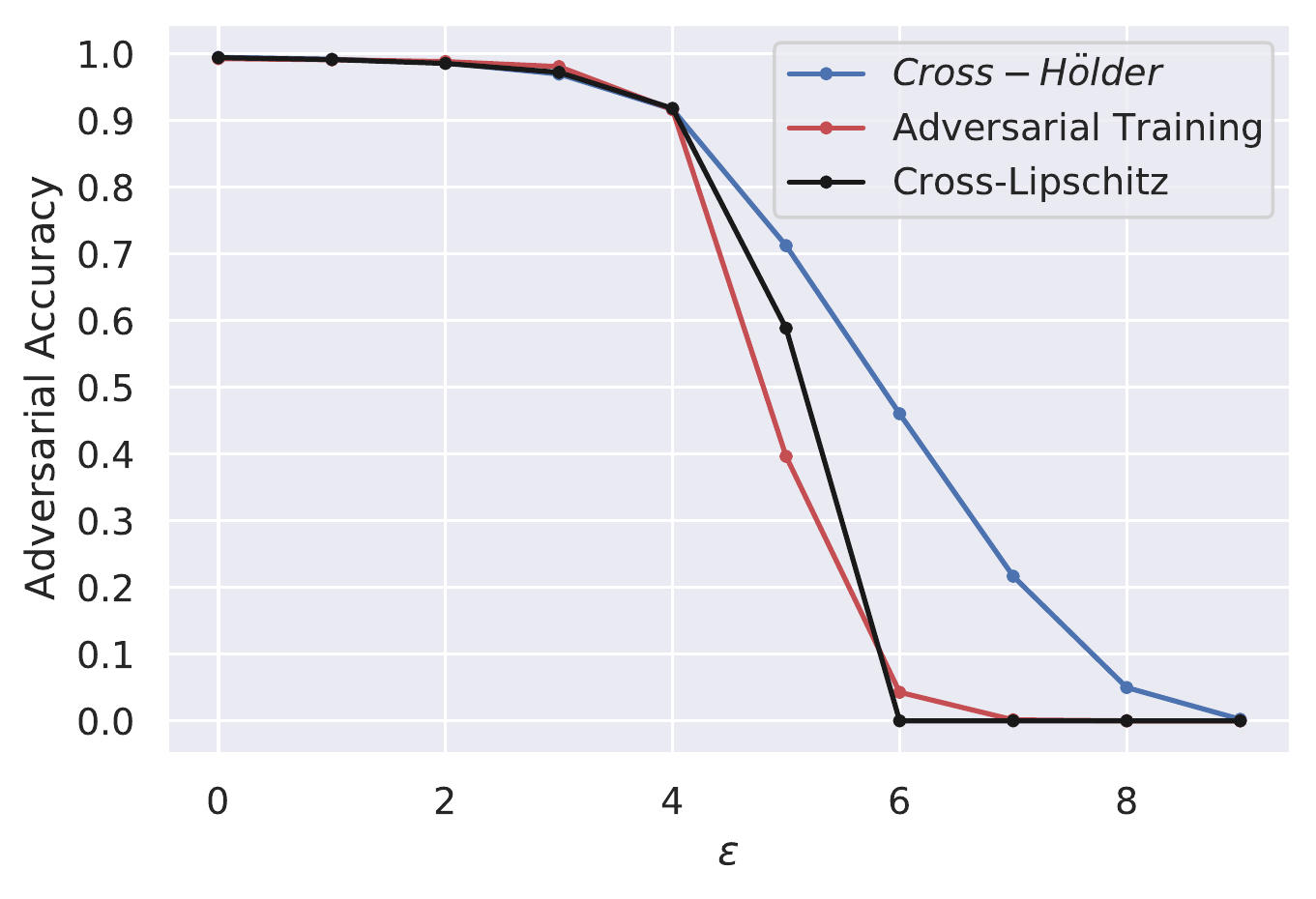}
\caption{Adversarial accuracy versus robustness levels for MNIST dataset.}
\label{fig:mnist_res}
\end{figure}
\begin{figure}[h!]
    \centering
 
    \label{fig:fmnist_mini}
    \includegraphics[width=0.4\textwidth]{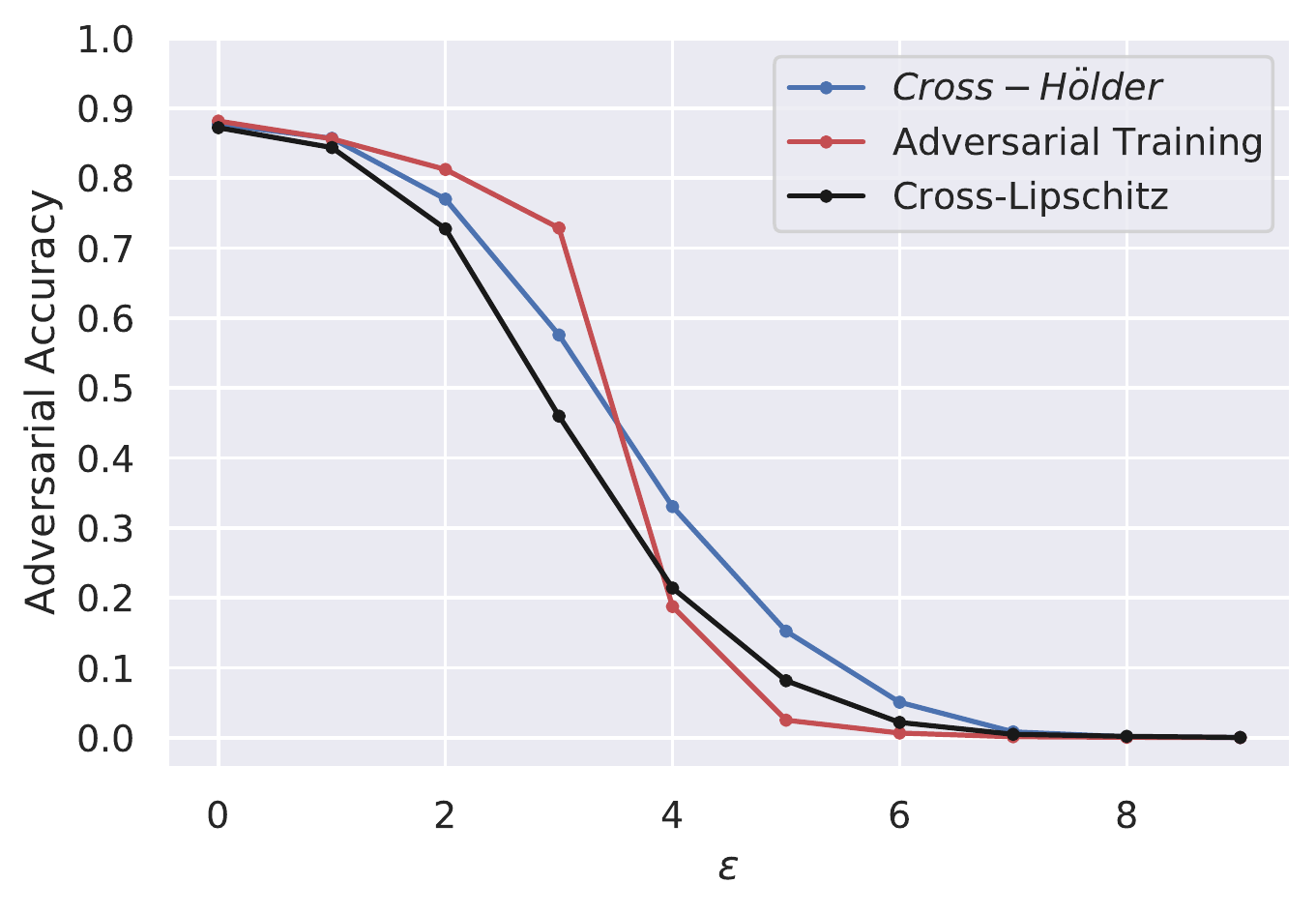}
    \caption{Adversarial accuracy versus robustness levels fashion MNIST dataset.}
    \label{fig:fmnist_res}
\end{figure}
Figure \ref{fig:fmnist_res} shows the worst case adversarial accuracy on the FMNIST dataset. Again Cross-H\"older outperforms Cross-Lipschitz but is inferior to adversarial training up to robustness level 4.0. Experiments on FMNIST showed the effectiveness of Cross-H\"older regularizer over Cross-Lipschitz while being competitive with adversarial training.
\section{Conclusion}
In this paper we present a method for an input gradient regularization of neural networks, which we call \emph{Cross-H\"older regularization}. We have proven both theoretically and empirically that this regularization is useful for increasing the resilience of deep neural networks against adversarial attacks. In future work, we wish to analyze the proposed regularization algorithm on further datasets.
\section*{Acknowledgement}
Marius Kloft acknowledges support by the German Research Foundation (DFG) award KL 2698/2-1 and by the Federal Ministry of Science and Education (BMBF) awards 01IS18051A and 031B0770E. Robert A. Vandermeulen acknowledges support by the Berlin Institute for the Foundations of Learning and Data (BIFOLD) sponsored by the German Federal Ministry of Education and Research (BMBF). The experiments were partly executed on the high performance cluster “Elwetritsch II” at the TU Kaiserslautern (TUK) which is part of the “Alliance for High Performance Computing Rheinland-Pfalz”(AHRP). We kindly acknowledge the support of the RHRK especially when using their DGX-2.
\bibliography{example_paper}

\bibliographystyle{icml2020}

\newpage
\clearpage

\appendix
\section{Related Work}
\label{sec:background}

Since the work of \cite{Szegedy2013IntriguingPO} and \cite{biggio2013evasion}, neural network models have been shown to be vulnerable to small input perturbations. Many elaborate attacks have been introduced since then \cite{Carlini2017TowardsET,brendel2017decision,Moosavi-Dezfooli2015}. As a response, many defenses---training algorithms that harden the attack problem---have also been devised \cite{Madry2017TowardsDL,Papernot2016DistillationAA,Hein2017FormalGO}. In this section, we give a brief overview of the prominent attacks and related defenses.

Attacks can be classified into \textit{black-box} \cite{brendel2017decision} and \textit{white-box} \cite{Carlini2017TowardsET} attacks. In a white-box attack, the attacker has full knowledge of the classifier used (e.g., the neural network architecture and its parameters). On the other hand, black-box attacks can only query a classifier but do not have the full knowledge of the model. Both types of attacks have the same goal: they are trying to solve the problem \cite{Szegedy2013IntriguingPO,Carlini2017TowardsET}:
\begin{equation}
    \label{opt:obj_norm}
        \min_\delta \|\delta\|_p 
        \quad \text{s.t. } \class(x) \neq \class(x+\delta),
\end{equation}
or alternatively \cite{Madry2017TowardsDL,Goodfellow2014ExplainingAH}
\begin{equation}\label{opt:obj_loss}
        \max_\delta \ell(x+\delta,y) \quad
        \text{s.t. } \|\delta\|_p \le \epsilon.  
\end{equation}
The adversarial attack problem is usually not convex due to the constraint in the problem \eqref{opt:obj_norm} and the objective of \eqref{opt:obj_loss}. 

In \citet{Madry2017TowardsDL} the authors argue that projected gradient descent (PGD) with multiple steps provides a good solution to the problem. They showed that, when PGD is run with different random restarts, the optimized objective values seem to concentrate. They then concluded that PGD does not get stuck in a bad local minimum. Therefore, they advised using PGD with multiple steps to craft adversarial examples and to evaluate the robustness of classifiers. The authors of \citet{Carlini2017TowardsET} attempt to solve the problem \eqref{opt:obj_norm} by a Lagrangian relaxation. They do not specify a robustness level $\epsilon$ constraint to the attack. However, they minimize the distance of the crafted adversarial example to the input while adding a margin like loss to derive the constraint of \eqref{opt:obj_norm} to be satisfied. A combination of the two approaches was used by the authors of \citet{Athalye2018ObfuscatedGG}. In it the authors used the PGD algorithm to solve a relaxed objective of \eqref{opt:obj_norm}. 

A more realistic setting for practical attacks is the black-box setting \cite{papernot2017practical,brendel2017decision}, in which the attacker can access the attacked model only through queries. In \citet{papernot2017practical} the authors proposed to query the attacked model with random samples and use the returned label to train a local replica model. They then craft adversarial examples on the replica in a white-box fashion. Those examples are then shown to transfer to the attacked model, that is, they are also misclassified under the black-box model. The authors of \citet{brendel2017decision} proposed a heuristic search for adversarial examples. Given an input example, they generate a random example with a different class. Starting from it, they sample a random direction to move according to a \textit{proposal} distribution. This distribution is designed such that moving in its sampled directions, the distance between the original and current point is reduced. After each move, the new point is checked if it still has a different class. If yes, the point is accepted, otherwise, it is rejected and a new sample direction is sampled. After a large number of iterations, they show that they could craft adversarial samples that could fool industrial deployed systems.   

As more attacks are devised, many defense mechanisms have been proposed. Defenses can be characterised into data augmentation \cite{Madry2017TowardsDL,kannan2018adversarial}, regularization based \cite{bietti2018kernel,Ciss2017ParsevalNI} and detection \cite{magnet2017,grosse2017statistical}.

Perhaps the most widely used defense mechanism is a data augmentation method known as \emph{adversarial training} \cite{Madry2017TowardsDL}. In adversarial training one generates adversarial examples using an attack and includes these examples in the training set. It was proposed concurrently with the introduction of adversarial examples in \citet{Szegedy2013IntriguingPO}. The proposed attack method in \citet{Szegedy2013IntriguingPO} was, however, slow and hence not suitable as an inner loop in the training procedures. A fast attack was proposed in \citet{Goodfellow2014ExplainingAH} to enhance adversarial training. It was, however, later shown that such a fast method is inadequate to promote real robustness \cite{tramer2017ensemble} as authors introduce the masking gradient phenomena. Masking gradient or more generally obfuscating gradients \cite{Athalye2018ObfuscatedGG} is a term coined to describe the situation where a defense causes the input gradient of a model to be non-informative. This prevents weak attacks from altering example label but stronger attacks like PGD would still be successful \cite{Athalye2018ObfuscatedGG}. When a strong attack used in adversarial training, real robustness can be achieved \cite{Athalye2018ObfuscatedGG,Madry2017TowardsDL}. A theoretical motivation to adversarial training was introduced in \citet{Madry2017TowardsDL} where the authors viewed it as solving a robust optimization problem
\begin{equation*}
    \min_\theta  \mathbb{E}[\max_{\|\delta\|_p \le \epsilon} \ell({\bf f}(x + \delta),y)].
\end{equation*}
They then utilize Danskin's theorem \cite{danskin2012theory} to argue that one can solve the inner maximization with PGD and substitute solution $\delta^*$ to solve the outer minimization. This argument, however, requires the inner maximization to be solved which can not be guaranteed due to the non-concavity of the problem. Adversarial training with a PGD has withstood a lot of tests \cite{Athalye2018ObfuscatedGG} and is therefore considered the state of the art. Different flavors of adversarial training have also been proposed. For example in \citet{kannan2018adversarial}, the authors proposed to use the distance between the \textit{logits} of clean and adversarial examples as a regularization term. Similarly, \citet{zhang2019theoretically} proposed to use a \textit{cross-entropy} loss between the predicted class probability of the model on a clean example and on adversarial examples generated based on that clean example. 

Another class of defenses is that of regularization based approaches. In this approach, one identifies a property of the model and relates it to adversarial vulnerability and designs a regularization term to reduce such a property. For example, in \citet{bietti2018kernel,Ciss2017ParsevalNI} the authors proposed to regularize the Lipschitz constant of the prediction model by reducing the operator norms of the weight matrices. The norm of the input gradient has been proposed to use as a regularizer to defend against adversarial examples \cite{Hein2017FormalGO,ross2018improving,finlay2019scaleable}. Interestingly in \citet{Hein2017FormalGO}, the authors derived a principled approach to defense. They theoretically derived a relationship between the norm of the adversarial noise and the maximum norm of the input gradient locally around an input example. They then proposed to regularize the norm of the gradient at the input point to promote robustness. Since the regularizer is designed to derive the norm of the gradient to be small at the input point and not the ball around it, it might be less powerful as there is an implicit assumption that the model is locally linear around the input. In our work, we propose a second order approach which provides a better local approximation of a model. We follow the principled approach of \citet{Hein2017FormalGO} by first deriving a relationship between adversarial robustness and the maximum operator norm of the Hessian matrix in a ball around an input example. Motivated by this result, we propose a regularizer to derive the operator norm of the Hessian to be small and thus promote robustness. 

Related to our proposed method is the work of \citet{moosavi2019robustness} and \citet{singla2020curvaturebased}. In \citet{moosavi2019robustness}, the authors experimentally noticed that adversarial training reduces the curvature of the loss function of neural networks, where they define curvature as the maximum magnitude of Hessian eigenvalues. They further proposed to regularize the curvature to induce robustness to the trained models. They used two main approximations in their regularizer. First, they used a finite difference method to compute a Hessian vector multiplication and therefore their proposed regularizer is simply a difference of gradients regularization. In our work, we used an exact and efficient Hessian vector multiplication algorithm \cite{Pearlmutter1994}. Second, the authors estimated the Hessian operator norm by the norm of the Hessian multiplied by the gradient assuming that the gradient and the vector associated with the largest eigenvalue of the Hessian are almost parallel. We experimentally show that it is not a tight lower bound on the operator norm of the Hessian. Therefore, we use an optimization approach to estimate the true operator norm of the Hessian. On the other hand \citet{singla2020curvaturebased} proposes a certificate, a computable lower bound on the norm of adversarial noise, on adversarial robustness based on the curvature. They provided an upper bound on the operator norm of the input Hessian of the model and proposed to use it as a regularizer. The upper bound however is independent of the input example and therefore not tight for many inputs as we show in the experiments section that the Hessian operator norm distribution has a large tail on real inputs let alone pathological ones.
\section{Proof of Theorem 2}
\label{sec:proofs}

In this section, we give the proof of theorem \ref{thm:tight}. We restate the theorem below and give its proof. 
\begin{theorem}
 
    Let $x \in \R^d$ and ${\bf f}: \R^d \rightarrow \R^2$ be a binary-classifier scoring function with continuously twice differentiable components. Further let ${\bf f}_0(x) > {\bf f}_1(x)$ and define $f = {\bf f}_0 - {\bf f}_1$. $H_f:\R^d\rightarrow \R^{d\times d}$  denotes the Hessian of $f$. For $\delta \in \mathbb{R}^d$ and $p,q$ such that $\frac{1}{p} + \frac{1}{q} = 1$, if ${\bf f}_1(x+\delta) > {\bf f}_0(x+\delta)$, then the following bound holds:
    \begin{align*}
        \|\delta\|_p \ge \max_{R>0}\min\left \{R,\frac{f(x)}{\|\nabla f(x)\|_q+ \frac{R}{2}K_R(f,x)}\right \},
    \end{align*}
     where $K_R(f,x) = \max_{\|\gamma\|_p\le R}K(f,x+\gamma)$, for $K(f,x):= \opn{H_f(x)}$.
\end{theorem}
\begin{proof}
  By Taylor theorem with remainder \cite{Tu2011} we have,
  \begin{equation*}
    f(x+\delta) = f(x) + \left<\nabla f(x), \delta\right> + \int_0^1 \delta^T H_{f}(x+t\delta)\delta(1-t) dt,
  \end{equation*}
  Given that ${\bf f}_1(x+\delta) \ge {\bf f}_0(x+\delta)$, we get $f(x+\delta) \le 0$ and hence, 
\begin{equation*}
    f(x) \le  \left<-\nabla f(x), \delta\right> + \displaystyle\int_0^1 -\delta^T H_{f}(x+t\delta)\delta (1-t) dt. 
\end{equation*}
By applying H\"older inequality, monotonicity of integration and the definition of the operator norm, we get,
\begin{equation}
\label{eq:bfor_max}
\begin{split}
    f(x) 
    \le&  \left<-\nabla f(x), \delta\right> + \int_0^1 -\delta^T H_{f}(x+t\delta)\delta (1-t) dt \notag\\
    \le&  \|\nabla -f(x) \|_q \|\delta\|_p \\&+ \int_0^1\|\delta\|_p \left\|H_{f}(x+t\delta)\delta\right\|_q (1-t) dt  \notag\\
    \le&  \|\nabla -f(x) \|_q \|\delta\|_p \\&+ \int_0^1\|\delta\|_p^2 \opn{H_{f}(x+t\delta)}(1-t) dt  \notag\\
    =& \|\nabla f(x) )\|_q \|\delta\|_p \\&+ \|\delta\|_p\int_0^1 \opn{H_{f}(x+t\delta) }\|\delta\|_p(1-t) dt. 
    \end{split}
\end{equation}
For a fixed $R>0$, assume that $\|\delta\|_p \le R$, then by monotonicity of integration we have, 
\begin{align*}
    &\int_0^1 \opn{H_{f}(x+t\delta) }\|\delta\|_p(1-t) dt\\ 
    &\le \displaystyle\int_0^1 \max_{\|\gamma\|_p \le R}\opn{H_{f}(x+\gamma) }\|\delta\|_p(1-t) dt\\
    &= K_R(f,x) R \displaystyle\int_0^1 (1-t) dt\\
    &=\frac{R}{2} K_R(f,x).
\end{align*}
Substituting in \eqref{eq:bfor_max} we get,
\begin{align*}
    f(x) 
    &\le  \|\nabla f(x) \|_q \|\delta\|_p +  \|\delta\|_p\frac{R}{2}K_R(f,x)\\
    &\le \|\delta\|_p \left(\|\nabla f(x) )\|_q  +  \frac{R}{2}K_R(f,x)\right).
\end{align*}
  Thus we have that
  \begin{equation*}
        \|\delta\|_p \ge \frac{f(x)}{\|\nabla f(x) \|_q  +  \frac{R}{2}K_R(f,x)}.
      \end{equation*}
          Since we restricted $\|\delta\|_p$ to be at most $R$, the last bound is only correct when the left-hand side is less than $R$. So we restrict the lower bound to be at most $R$. That is,
  \begin{equation*}
        \|\delta\|_p \ge \min\left\{ R,\frac{f(x)}{\|\nabla f(x) \|_q  +  \frac{R}{2}K_R(f,x)}\right\}.
      \end{equation*}
  Maximizing of $R$ we get,
  \begin{equation*}
        \|\delta\|_p \ge \max_{R > 0}\min\left\{ R,\frac{f(x)}{\|\nabla f(x) )\|_q  +  \frac{R}{2} K_R(f,x)} \right\}.
      \end{equation*}
\end{proof}



\section{Experiments}
\label{sec:det_exp}
In this section, we give a detailed experimental setup and results. We begin by investigating the estimation of the Hessian operator norm and compare this to the norm of the Hessian multiplied by the gradient to the Hessian operator norm. We then present experiments where we evaluate adversarial robustness on both the MNIST and FMNIST datasets using our adversarial defense versus gradient based methods and adversarial training. 

\subsection{Hessian Operator Norm Estimation} \label{sec:opn_exp}
In this experiment we investigate the Hessian operator norm approximation presented in \citet{moosavi2019robustness}. The authors propose to estimate the Hessian\footnote{The Hessian is computed on $f^{(j)}(x)$, where $j$ is the index of the second largest logit.} operator norm with the norm of the Hessian multiplied by the gradient, we denote this quantity by $|Hg|$. We compare this to the true Hessian operator norm computed by gradient ascent optimization that we denote by $|Hv^*|$. We train a convolution neural network with the architecture below on MNIST dataset with no defense. Then for each data point on the evaluation set we compute both $|Hg|$ and $|Hv^*|$ and plot the histogram (see figure \ref{fig:exp_op_est}). Figure \ref{fig:opn_his} shows that the histogram of $|Hv^*|$ is shifted to the right and has a heavier tail. There is however a large overlap between the two histograms. We then investigate whether this overlap correspond to $|Hg|$ being a good estimate or not. Figure \ref{fig:opn_diff} depicts the histogram of the difference $|Hv^*| - |Hg| $ where we notice that it has a mode at around 5 with a heavier right tail. Therefore, we conclude that $|Hg|$ is not a tight lower bound and hence it is better to use the true operator norm for regularizing the Hessian. On the other hand, we see from the heavy tail of the histogram of $|Hv^*|$ that any global bound on the Hessian operator norm is not tight for most of the input examples and therefore using it as regularizing term as in \cite{singla2020curvaturebased} is restrictive, especially when the true operator norm can be efficiently used as in our approach.

\subsection{Adversarial Robustness Evaluation}
In this section, we evaluate the adversarial robustness of models trained with the Cross-H\"older regularizer. We test the robustness on two datasets, namely MNIST and FMNIST.
\subsubsection{Setup}
 As baseline defenses, we use the gradient based approach of Cross-Lipschitz \cite{Hein2017FormalGO}. We also compare against adversarial training \cite{Madry2017TowardsDL} as it is considered state of the art defense \cite{Athalye2018ObfuscatedGG}. We used the same network architecture for both the MNIST and FMNIST datasets which is the same architecture used in \citet{Carlini2017}. It consists of 4 convolution layers with $3\times3$-sized filters and the respective filter maps sizes of (32,32,64,64), followed by two fully connected layers of size 200 connected to the logits layer with size 10.
 We introduced two changes to the architecture to ensure that it is twice differentiable. First, we replaced the ReLU activation function with SWISH $x\sigma(x)$ \cite{Ramachandran2017SearchingFA}, where $\sigma$ is the sigmoid function. Second, max-pooling was replaced by strided convolution every second convolutional layer. 
\\\\
\noindent \textbf{Robustness Evaluation}
Evaluating the robustness is a hard problem due to the high non-convexity of the attack optimization problem \cite{Madry2017TowardsDL}. Many defenses claim robustness against weak attacks while they fail to be robust against harder attacks \cite{Athalye2018ObfuscatedGG,carlini2019evaluating}. Therefore it is important to use the strongest available attacks to evaluate the robustness of models trained with a given defense mechanism. We use the recommended PGD attack with multiple iterations and multiple random initializations \cite{carlini2019evaluating}. Since in adversarial training an attack is used in the training procedure it is recommended to use a different attack in the evaluation phase \cite{carlini2019evaluating}. Therefore we use two iterative PGD attacks with 50 iterations and 10 random restarts. The first is using cross entropy loss as objective, the same as adversarial training. The second is using the Carlini and Wagner (CW) objective \cite{carlini2019evaluating,Athalye2018ObfuscatedGG}. We also report the worst case results of the two attacks, which is the minimum adversarial accuracy of the two attacks. The robustness metric we report is adversarial accuracy, the accuracy after we have applied a given attack. We report adversarial accuracy for different robustness levels. 
\\\\
\noindent \textbf{Hyperparameters} 
We implement three defenses that depend on several hyperparameters. Adversarial training has the parameter $\varepsilon$, the robustness level used for the training attack.  Cross-H\"older has the parameters $\lambda_1$ and $\lambda_2$. We choose the hyperparameters by leaving out 5\% of the data outside of the training set. We employ a grid search to select the hyperparameter on the left out evaluation set. The criteria to select the winning hyperparameter is for it to achieve at least 99\% accuracy on the clean evaluation set in case of MNIST and 88\% in the case of FMNIST. We then select the hyperparameter by selecting the best sum of adversarial accuracy on robustness levels 3.0 and 4.0. The selected parameters for adversarial training were $\epsilon = 5.0$ on MNIST and $\epsilon = 4.0$ on FMNIST. For Cross-Lipschitz, we selected $\lambda_1 = 0.2$ for MNIST and $\lambda_1= 0.15$ for FMNIST. Finally for Cross-H\"older, we select $\lambda_1=0.02$, $\lambda_2=0.5$ for MNIST and $\lambda_1=0.2$, $\lambda_2=0.5$ for FMNIST.
\\\\
\noindent \textbf{Results} 
Figure \ref{fig:mnist_res} shows the result of adversarial evaluation on the MNIST dataset. Figure \ref{fig:mnist_pgd_xent} shows the result of applying the PGD attack with the cross entropy loss. The figure shows that all three defenses are quite effective up to robustness level 3.0 after which we see that both Cross-Lipschitz and Cross-H\"older accuracies decrease with Cross-Lipschitz decreases faster. This was expected as in adversarial training the same attack was used in training. Figure \ref{fig:mnist_pgd_cw} shows the results of applying PGD with CW loss. On this attack Cross-H\"older behaves better than both adversarial training and Cross-Lipschitz after robustness level 4.0. Figure \ref{fig:mnist_worst} shows the result of the worst case attack which again shows that Cross-H\"older outperforms both Cross-Lipschitz and adversarial training after robustness level 4.0. On MNIST, our experiments suggest that Cross-H\"older regularizer outperforms both Cross-Lipschitz and adversarial training. 
\begin{figure}
    \centering
    \subfloat[ \label{fig:opn_his}]{{\includegraphics[width=0.4\textwidth]{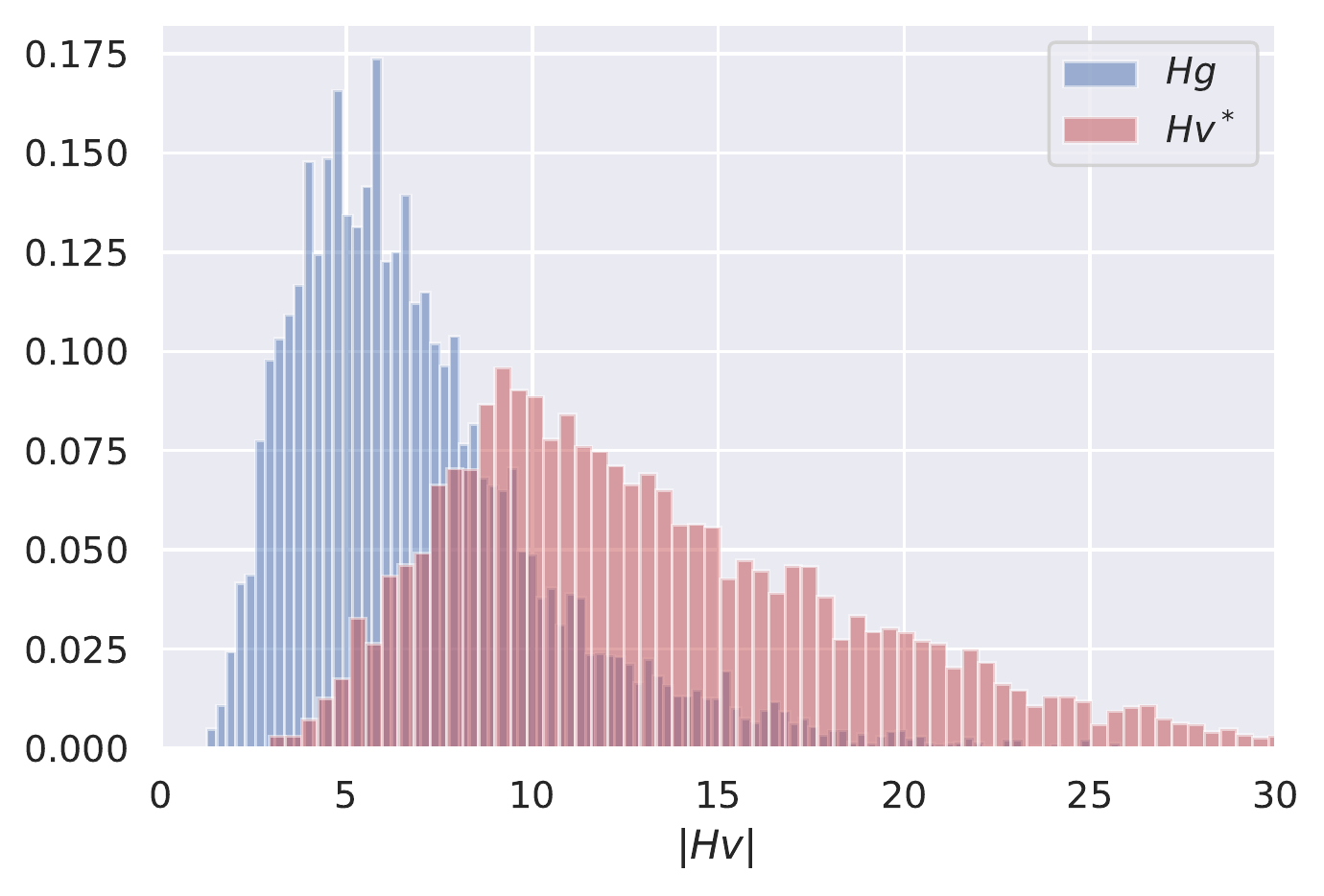}}}
    \qquad  
    \subfloat[\label{fig:opn_diff}]{{\includegraphics[width=0.4\textwidth]{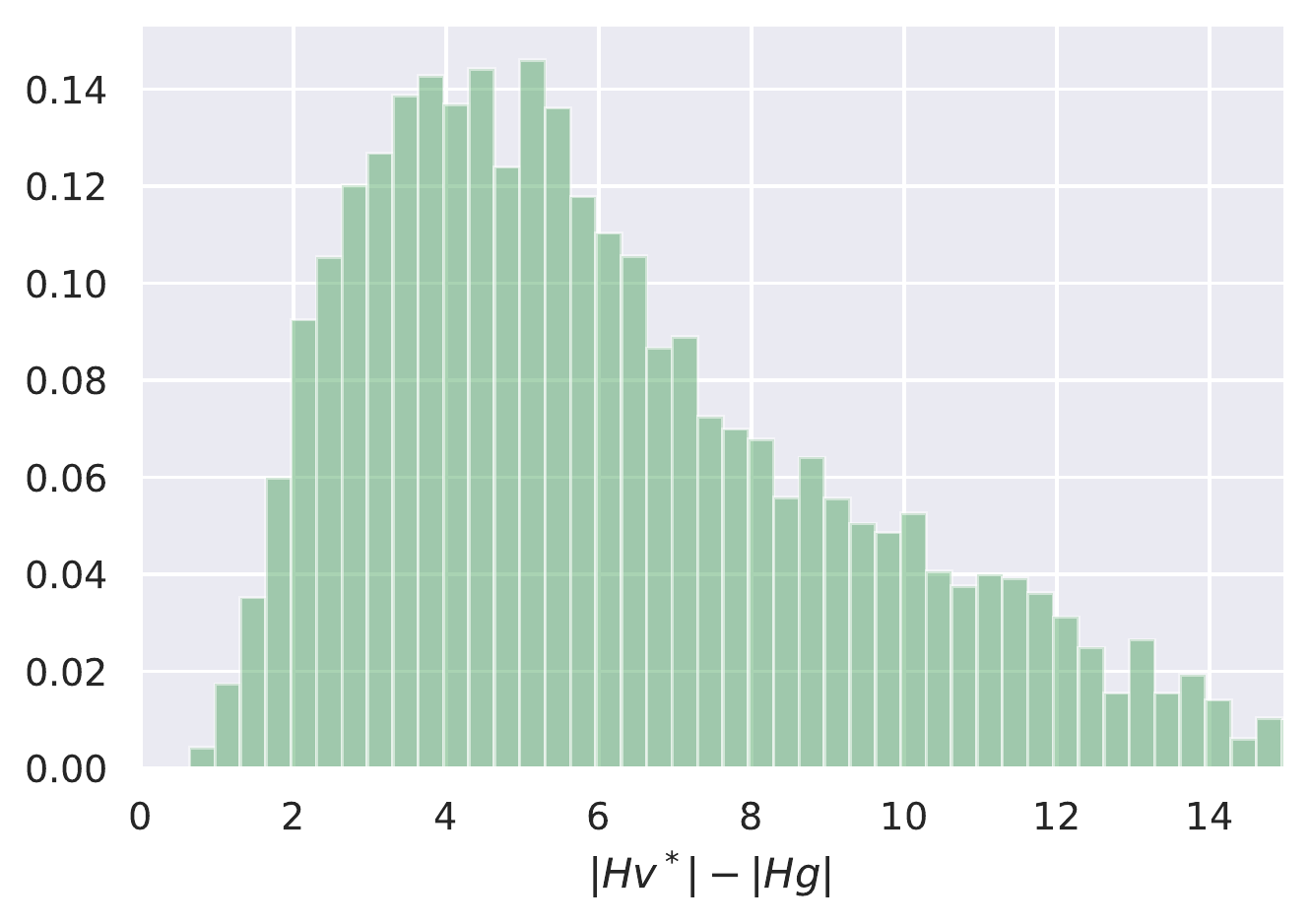}}}
    \caption{Hessian operator norm estimation comparison.}
    \label{fig:exp_op_est}
\end{figure}
\clearpage
\begin{figure}
 \centering
    \subfloat[PGD on cross entropy loss.\label{fig:mnist_pgd_xent}]{{\includegraphics[width=0.4\textwidth]{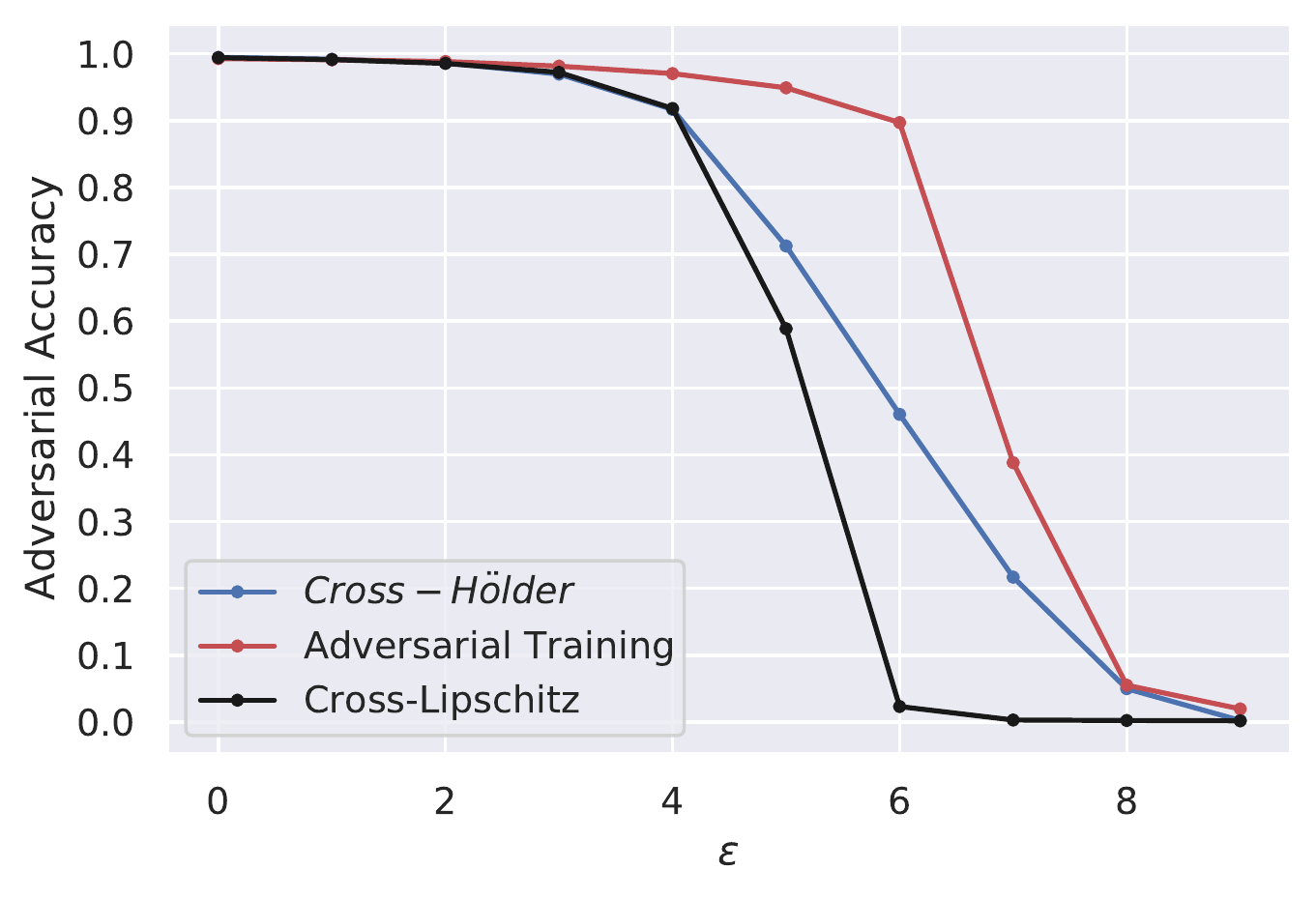}}}
    \qquad  
    \subfloat[PGD on CW loss. \label{fig:mnist_pgd_cw}]{{\includegraphics[width=0.4\textwidth]{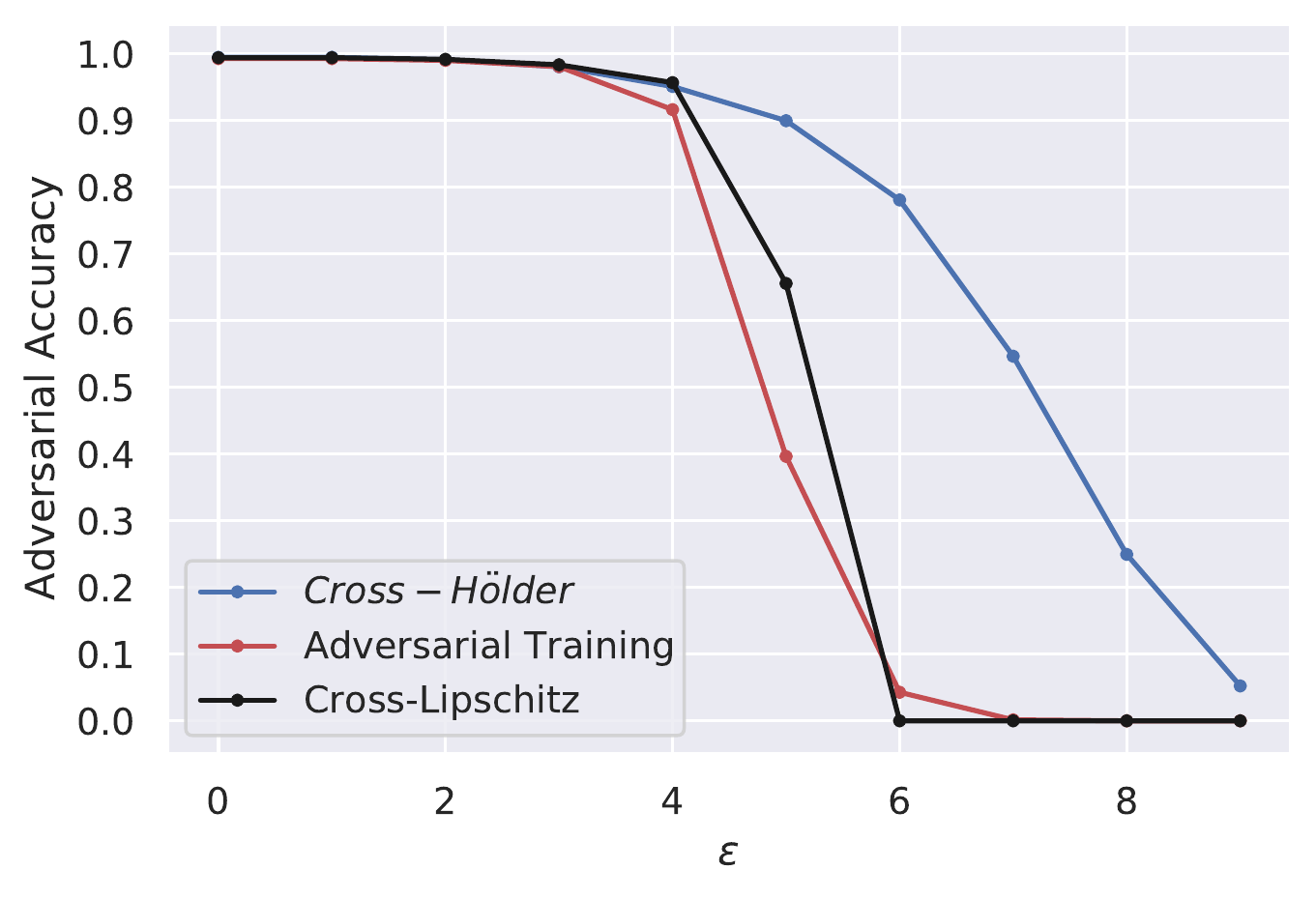}}}
    \qquad
    \subfloat[Worst case attack.\label{fig:mnist_worsta}]{{\includegraphics[width=0.4\textwidth]{pgd_mini.pdf}}}
\caption{Adversarial accuracy versus robustness levels for MNIST dataset.}
\label{fig:mnist_resa}
\end{figure}
\begin{figure}
    \centering
    \subfloat[PGD on cross entropy loss.\label{fig:fmnist_pgd_xent}]{{\includegraphics[width=0.4\textwidth]{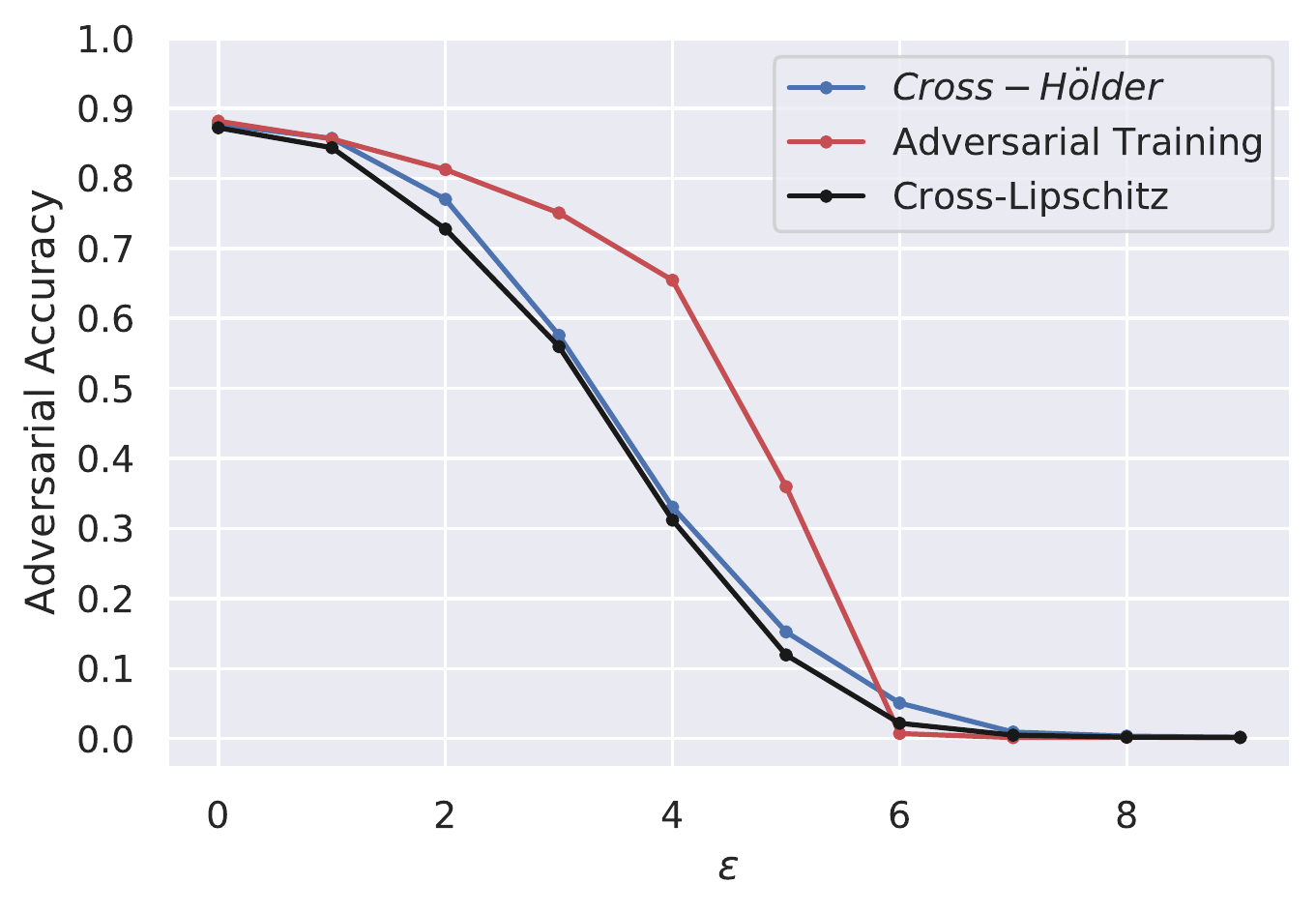}}}
    \qquad  
    \subfloat[PGD on CW loss.  \label{fig:fmnist_pgd_cw}]{{\includegraphics[width=0.4\textwidth]{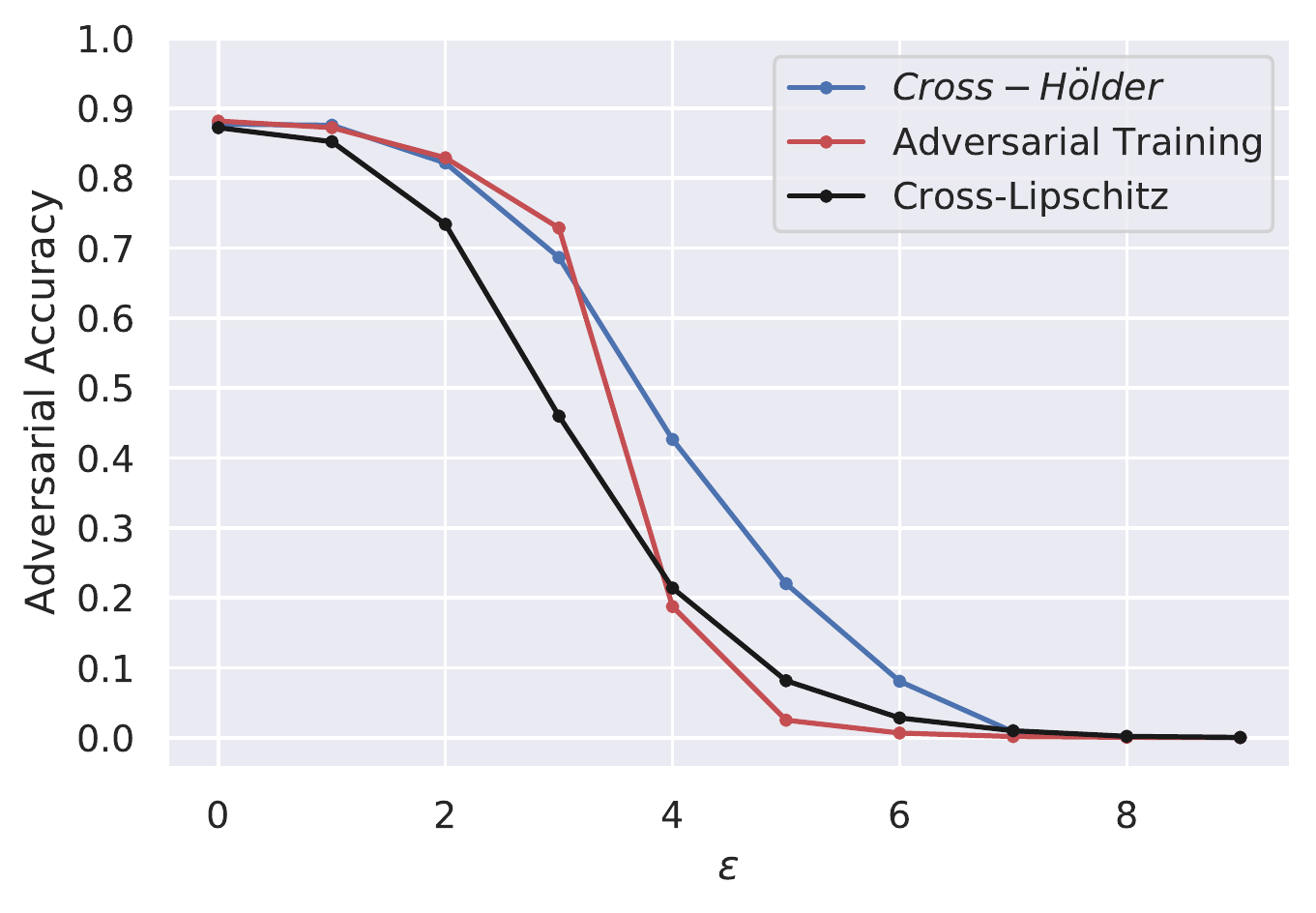}}}
    \qquad
    \subfloat[Worst case attack.\label{fig:fmnist_minia}]{{\includegraphics[width=0.4\textwidth]{fmnist_pgd_mini.pdf}}}
    \caption{Adversarial accuracy versus robustness levels fashion MNIST dataset.}
    \label{fig:fmnist_resa}
\end{figure}

\end{document}